\documentclass[accepted]{uai2025}

\usepackage[american]{babel}

\usepackage{natbib} 
\usepackage{mathtools} 

\usepackage{booktabs} 
\usepackage{tikz} 
\usepackage{amsmath, amsthm, hyperref, amsfonts, algorithm, algpseudocodex}
\usepackage{siunitx}

\newtheorem{theorem}{Theorem}
\newtheorem{lemma}{Lemma}
\newtheorem{assumption}{Assumption}
\newtheorem{definition}{Definition}
\newtheorem{remark}{Remark}

\newcommand{\converge}[1][]{\xrightarrow{#1}}

\newcommand{\R}{\mathbb{R}}

\newcommand{\pr}{\mathbb{P}}
\newcommand{\E}{\mathbb{E}}

\DeclareMathOperator*{\argmin}{argmin}

\title{Targeted Learning for Variable Importance}

\author[1]{\href{mailto:<xw547@cornell.edu>?Subject=Your UAI 2025 paper}{Xiaohan Wang}{}}
\author[2]{Yunzhe Zhou}
\author[3]{Giles Hooker}

\affil[1]{
Department of Statistics and Data Science\\
Cornell University\\
Ithaca, New York, USA
}
\affil[2]{
Department of Biostatistics\\
University of California, Berkeley\\
Berkeley, California, USA
}
\affil[3]{
Department of Statistics and Data Science\\
University of Pennsylvania\\
Philadelphia, Pennsylvania, USA
  }
  
\begin{document}
\maketitle

\begin{abstract}
  Variable importance is one of the most widely used measures for interpreting machine learning with significant interest from both statistics and machine learning communities. However, attention has only recently been directed toward uncertainty quantification in these metrics. Current approaches largely rely on one-step procedures, which, while asymptotically efficient, can present higher sensitivity and instability in finite sample settings. To address these limitations, we propose a novel method by employing the {\it targeted learning} (TL) framework, designed to enhance robustness in inference for variable importance metrics. Our approach is particularly suited for conditional permutation variable importance. We show that it (i) retains the asymptotic efficiency of traditional methods, (ii) maintains comparable computational complexity, and (iii) delivers improved accuracy, especially in finite sample contexts. We further support these findings with numerical experiments that illustrate the practical advantages of our method and validate the theoretical results.
\end{abstract}

\section{Introduction}
Machine Learning (ML) models offer high-quality predictions for complex data structures and have become indispensable across various fields, including civil engineering \citep{lu2023using}, sociology \citep{molina2019machine}, and archaeology \citep{bickler2021machine}, due to their versatility and predictive power. However, due to their complexity, humans find the internal structures of ML models challenging to turn into real-world interpretation \citep{hooker2017machine, hooker2021unrestricted, freiesleben2024scientific}. To address this, a considerable suite of post hoc interpretable machine learning (IML) tools have been developed. 

Among these tools, variable importance, which measures the contribution of individual covariates to the response variable, is a widely adopted measure in IML \citep{molnar2020interpretable}. Traditionally, this has been applied to assess the behavior of fixed models, such as random forests \citep{breiman2001random} and linear models \citep{gromping2007estimators}. Additionally, efforts have been made to create model-specific uncertainty quantification methods, as seen in \citet{gan2022model}. Building on these advances, there is a growing interest in exploring model-agnostic variable importance using nonparametric techniques \citep{van2006statistical, lei2018distribution, williamson2021nonparametric, donnelly2023rashomon,verdinelli2024decorrelated}.

Despite substantial efforts devoted to developing new methodologies, little attention has been given to fully understanding these tools. Specifically, there are few methodological developments around  uncertainty quantification for variable importance metrics. Some recent work has started to develop such methods, \citep{williamson2021nonparametric, williamson2023general,wolock2023nonparametric,freiesleben2024scientific,fauvel2025sobol}. However, these have focused on utilizing one-step de-biasing procedures.

In this paper, we introduce a novel method to quantify the uncertainty of variable importance metrics. Employing the targeted learning framework of \citet{van2006statistical}, our method provides a robust algorithm for conducting inference on variable importance. Our approach is statistically efficient within the class of regular estimators as well as computationally cheap. Particularly, we focus on conditional permutation importance, as this method avoids potential issues with extrapolation \citep{hooker2021unrestricted}.

This paper is organized as follows: In Section \ref{sec: variable importance}, we formally state the problem setup and introduce some key concepts related to variable importance. In Section \ref{sec: methodology}, we give an overview of the existing methodology and its justification, present our methodology, and illustrate it using conditional permutation importance. In Section \ref{sec: theory}, we introduce the efficiency theory and the theoretical guarantees of our methodology. Lastly, we illustrate the effectiveness of our method through simulation studies and two real-world data applications. We make our code publicly available at \url{https://github.com/xw547/TL4VI}.

\section{Variable Importance}\label{sec: variable importance}
\subsection{Problem Setup}
Suppose that we observe \(n\) independent and identically distributed (i.i.d.) observations \( \{(Y_i, X_i, Z_i)\}_{i=1}^{n} \) drawn from the unknown joint distribution \( P^*_{Y, X, Z} \in \mathcal{M}\), where \(\mathcal{M}\) is the class of nonparametric distributions. That is, 
\[
(Y_i, X_i, Z_i) \overset{\text{i.i.d.}}{\sim} P_{Y,X,Z}, \quad i = 1, \dots, n.
\]

We aim to investigate the relationship between the response \( Y \in \R\) and the covariate of interest \( X  \in \mathcal{X}\) in the presence of other covariates \( Z \in \mathcal{Z} \) through some predefined variable importance, make sure it is ``efficient'', and then conduct inference. We define variable importance with respect to performance on some loss \(L(\cdot)\),  but our estimator  \(\hat{f}: \mathcal{X} \times \mathcal{Z} \to \mathbb{R}\) is assumed to approximate \(E(Y|X,Z)\). For simplicity, we'll focus on the case where \(\mathcal{X}\subseteq \R, \mathcal{Z} \subseteq \R^{d-1}\), yet we note that our method can be generalized to the cases where \(X_i\) is a vector. 

\subsection{Notation}
We use \(\pr_n\) to denote the empirical measure, that is, suppose \(f: \mathcal{X} \to \R \), \(\pr_n(f) = \frac{1}{n}\sum_{i=1}^n f(X_i)\). In contrast, we use \(\pr\) to denote the probability measure, that is, \(\pr(f) = \int f(X) d\pr\). And \(L_2^0(P)\) denotes the collection of functions such that \(P f =0 \) and \(Pf^2 < \infty\). \(O_P\) and \(o_P\) are used as follows: \(X_n = O_P(r_n)\) denotes \(X_n/r_n\) is bounded in probability and \(X_n = o_P(r_n)\) indicates \(X_n/r_n \converge[P] 0\), respectively. Lastly, we denote the \(L_2(P)\) norm as \(\|\cdot\|\).

\subsection{Variable Importance}\label{sec: vi}

A number of variable importance metrics have been suggested. Here we include a brief description of some of the most commonly studied. 

\subsubsection{Permutation Importance}
Variable importance is obtained by considering the out-of-bag (OOB) loss of a certain feature \citep{breiman2001random}. First, we permute the feature(s) that we are interested in quantifying the importance. That is, we randomly permute the index of the column of \(X\), denoted by \(X^\pi\), and then put it together with the remaining features, which results in the final data \((Y, X^\pi, Z)\). Then, for model \(\hat{f}: \mathcal{X} \times \mathcal{Z} \to \mathbb{R}\):
\begin{align*}
  VI^{\pi}_X  = \frac{1}{N}\sum_{i=1}^N L(Y_i, \hat{f}(X_i^\pi, Z_i) - L(Y_i, \hat{f}(X_i, Z_i))).
\end{align*}

While providing a starting point, this metric has been critiqued in \citet{strobl2008conditional,hooker2021unrestricted} as resulting in extrapolation when $(X_i^\pi, Z_i)$ are far from observed data.

\subsubsection{Conditional Permutation Importance}
This metric is obtained by conditional permutation copy of \(X\) such that: \(X_{i}^C \sim X_i|Z_i\), \(X_i^C\perp Y_i|Z_i\). With a similar notation defined above, we may thus have the plug-in estimator, defined as:
\begin{align*}
VI_X^C  = \frac{1}{N}\sum_{i=1}^N L(Y_i, \hat{f}(X_i^C, Z_i)) - L(Y_i, \hat{f}(X_i, Z_i)).
\end{align*}

This approach was first proposed by \cite{strobl2008conditional} for the random forest, where they obtain the conditional permuted version by conducting the permutation within each leaf. A similar idea is also present in \citet{fisher2019all,chamma2024statistically}. \citet{hooker2021unrestricted} observes that both conditional permutation, as well as Leave-One-Covariate-Out (LOCO) and other retraining methods, have the same population estimand that serves as our target. 

\subsubsection{Leave-One-Covariate-Out}
LOCO can be considered a nonparametric extension of the classical \(R^2\) statistic \citep{williamson2023general}. In addition to \(\hat{f}\), we training another model \(\hat{f}_{-X}:\mathcal{Z} \to \mathbb{R}\), which has no access to \(X\). The plug-in estimator is defined as:
\begin{align*}
VI_X^d  = \frac{1}{N}\sum_{i=1}^N L(Y_i, \hat{f}_{-X}(Z_i) - L(Y_i, \hat{f}(X_i, Z_i))).
\end{align*}
This approach was first proposed by \cite{lei2018distribution}, and \cite{williamson2023general} quantified the uncertainty of the method through the efficient influence function. See \citet{mentch2022getting} for cautionary results.

\section{Methodology}\label{sec: methodology}
Existing literature on quantifying the uncertainty of variable importance is scarce. There are two main trends in the uncertainty quantification of IML.

The first is a de-biasing approach utilizing the influence function, in which a bias correction and confidence intervals are constructed from a one-step method \citep{williamson2023general,wolock2023nonparametric}. In \citet{williamson2021nonparametric} and \citet{williamson2023general}, the efficient influence function is leveraged to construct confidence intervals, since the plug-in estimator is shown to be efficient under mild assumptions. While \citet{wolock2023nonparametric} applies this concept to de-bias the variable importance for survival analysis, and then construct the confidence interval, an approach also seen in \citet{ning2017general,10.1111/ectj.12097}. 

Alternatively, \citet{molnar2023relating} and \citet{freiesleben2024scientific} propose a bootstrap-like approach for uncertainty quantification. In these methods, models are refitted on different subsets of the data, and the variance is estimated from the ensemble of models and their associated metrics. This approach is conceptually similar to the bootstrap variance estimation described by \citet{diciccio1996bootstrap}. However, these bootstrap-based methods require significant computational effort.

In the following subsections, we first review efficient influence function. Next, we briefly introduce the theory behind the de-biasing approach. We then formally present the proposed methodology within the targeted learning framework. Lastly, we present an implementation of the proposed methodology.

\subsection{Efficient Influence Function}
Influence functions characterize the first-order behavior of pathwise differentiable functionals. By naively appending the empirical estimator of the influence function, we can ``de-bias'' the estimator, which will be discussed in detail in the next section. 

\begin{definition}\label{definition: pairwise differentiable}
Let $\mathcal{M}$ be a class of probability distributions and $\Psi: \mathcal{M} \to \mathbb{R}$ be a functional. We say that $\Psi$ is pathwise differentiable at $P^* \in \mathcal{M}$ with tangent space \(\dot{P}_0\) if there exists a bounded linear function ${\psi}_{P^*}$, called the influence function, such that for \(P_{\varepsilon,g} = (1 + \varepsilon) P^* + \varepsilon g \in \dot{P^*}\) , the following holds:
\[
\left. \frac{d}{d\epsilon} \Psi(P_{\epsilon, g}) \right|_{\epsilon = 0} = \mathbb{E}_{P^*} \left[ {\psi}_{P^*}(X) \cdot \frac{d}{d\epsilon} \log \frac{dP_{\epsilon, g}}{dP^*}(X) \bigg|_{\epsilon = 0} \right].
\]
\end{definition}

Following \cite{hines2022demystifying}, for many targets $\Psi$ the influence function can be calculated from a G\^{a}teaux derivative in the direction of a point-mass contamination at each $x$:
\[
\psi_P(x) = \frac{d}{d\epsilon}  \Psi\left( (1-\epsilon)P + \epsilon \delta_{X=x} \right).
\]

For distributions \(\hat{P}, P^* \in \mathcal{M}\), if \(\Psi\) is a pathwise differentiable functional, we can consider the von Mises expansion:
\begin{align}\label{equation: von Mises}
  \Psi(\hat{P}) - \Psi(P^*) 
  &= \int \psi_{\hat{P}}(x) d(\hat{P} - P^*)(x) + R_2(\hat{P}, P^*),
\end{align}
where \(R_2(\hat{P}, P^*)\) is the second-order remainder term. Intuitively, the von Mises expansion can act as a distribution version of a Taylor expansion. In particular, if \(\mathcal{M}\) is the class of nonparametric distributions, the influence function \(\psi\) is also the efficient influence function \citep{hines2022demystifying, kennedy2022semiparametric}. The efficient influence function characterizes the optimal attainable asymptotic variance (See lemma 25.19 of \cite{van2000asymptotic} and theorem 5.2.1 of \cite{bickel1993efficient}.

\subsection{De-biasing approach}
Based on the above explanation, a natural idea would be to seek a bias correction. A naive bias correction estimator is:
\begin{align*}
  \hat\Psi_{naive} = \Psi(\hat{P}) + \pr_n ( \psi_{\hat{P}}).
\end{align*}

With the same von Mises expansion, we may then have:
\begin{align*}
  \hat\Psi_{naive}  -  \Psi({P}^*) 
  & = (\pr_n - \pr) (\psi_{P^*}) \\
  & +(\pr_n - \pr)(\psi_{\hat{P}} - \psi_{P^*})\\
  & + R_2(\hat{P}, P^*)
\end{align*}

The first term is a simple average of fixed functions, where we can then apply the central limit theorem. 
The second term is usually referred to as the {\em empirical process term}. If \(\psi_{\hat{P}} \converge[P] \psi_{P^*}\), it can be shown to be of order \(o_P(1/\sqrt{n})\) under either Donsker class assumptions on $\hat{P}$, or in the sample-splitting regime which we adopt. The last term, also called {\it second order term}, is generally assumed to be of order \(o_P(1/\sqrt{n})\), which is typically determined in a case-by-case manner \citep{cheng1984strong, luo2016high, benkeser2016highly,farrell2018deep, wei2023efficient}. 

\citet{wolock2023nonparametric} employs this first-order correction to provide uncertainty quantification for variable importance of survival analysis. 
This gives rise to the construction of a confidence interval from 
\[
\hat\Psi_{naive} \pm z_{\alpha/2}  s_{\pr_n}(\psi_{\hat P})
\]
in which $z_{\alpha/2}$ are the quantile of a normal distribution and $s_{\pr_n}(\psi_{\hat P})$ indicates the standard deviation over the values of the influence function. However, we note from the non-asymptotic perspective, that the instability of empirical distributions can hinder the effectiveness of both methods, as highlighted by \cite{booth1998monte} and \citet{van2011targeted}.

\subsection{Proposed Methodology}\label{sec: proposed methodology}

In contrast to current de-biasing methods, our method provides an iterative update to remove the bias and produces a more refined estimator than the one-step versions. 
The targeted learning framework, first proposed by \cite{van2006targeted}, originates from semiparametric statistics and causal inference.  To obtain an asymptotically linear estimator, they proposed to perturb the empirical distribution in the direction of the influence function to obtain an efficient estimator. 
Specifically, we create a one-dimensional family of densities starting from $\hat{P}$ and moving in the direction of the influence function: \(P_\varepsilon = (1+\varepsilon\hat\psi)\hat{P}\) and find the maximum likelihood estimate of $\varepsilon$:
\begin{align*}
\hat{\varepsilon} = \argmin_{\varepsilon} \pr_n \log P_\varepsilon. 
\end{align*}
This defines a new estimated density $\hat{P}_{\hat{\varepsilon}}$ for which we can calculate an efficient influence function, a new one-dimensional family and a corresponding update. Repeating this yields a sequence of estimates:
\begin{align*}
\hat{\varepsilon}^{j+1} & = \argmin_{{\varepsilon}_j} \pr_n \log P^j_{\varepsilon} \\
P^{j+1}_{\varepsilon} & = (1 + \varepsilon^{j+1} \hat \psi_{P^j}) P^j
\end{align*}
We continue updating the distribution until we obtain and update \(\hat{\varepsilon}^k = 0\) at iterate $k$ and obtain  a final debiased estimator $\hat{\Psi}_{n} = \Psi(P^k)$. Many extensions of this framework have been proposed for causal inference: cross-validation TL, dynamic treatment regimes, and time-to-event outcome \cite{van2011cross,luedtke2016optimal,cai2020one}.
Instead, we employ TL to conduct uncertainty quantification for variable importance.

The intention of this iterative definition is to ensure that the likelihood is maximized and that the plug-in bias term (\(\pr_n\psi(\hat{P}))\) is zero.  From this, the first order error of $\hat{\Psi}_n - \Psi(P^*)$ is described by $\pr_n \psi_{P^k}$ which admits a central limit theorem and in common with the naive method we base confidence intervals on the standard deviation  $s_{\pr_n}(\psi_{P^k})$. The improved accuracy of this framework over the naive implementation is due both to a more exact control of bias and because the naive method does not account for uncertainty due to the plug-in bias. The iterative scheme that we propose requires a single representation of the distribution $P$ and thus cannot directly be employed within LOCO-type models that rely on re-training estimators.

We note that in order to obtain theoretical results with weaker conditions, we adopted the sample splitting strategy  implemented in \citet{van2011cross}, \citet{10.1111/ectj.12097} and \cite{newey2018cross}. That is, the plug-in estimate is obtained from the first set of data \(I_1\) and the iterative update is conducted using an independent data set \(I_2\). In theoretical results, we also assume a third set \(I_3\) used to quantify uncertainty, although we do not believe this is strictly necessary. In the next section, we present an implementation of the proposed methodology through CPI.

\subsection{Illustration}
In this section we apply our methodology to conditional variable importance. Detailed steps of our implementation can be found at Algorithm \ref{alg: conditional permutation importance calculation}. We begin by observing that the estimand of conditional permutation metric is defined as:
\begin{align}\label{equation: definition vi}
  \Psi^{C}(X, Y, Z) = \E\left[  L(y, \hat{y}(X^C, Z)) - L(y, \hat{y}(X, Z))  \right],
\end{align}
where \(\hat{y}(x, z) = \E\left[ Y| X = x,Z = z \right]\), \(X^C \sim X|Z\), and  \(X^C \perp X|Z\).

Here the first term  measures the ``conditional permuted'' performance. The second term in the estimand is the reference loss, which serves as the ``benchmark'' of our importance. We therefore decompose the conditional permutation importance into:
\begin{align*}
  \Psi^{C}(X, Y, Z) &= \Psi^{C}_0(X, Y, Z)- \Psi_0(X, Y, Z), 
\end{align*}
where \( \Psi^{C}_0(X, Y, Z) = \E\left[L(y, \hat{y}(X^C, Z))\right]\) and  \(\Psi_0(X, Y, Z) =  \E\left[ L(y, \hat{y}(X, Z))  \right]\).

These each have a corresponding influence function:

\begin{lemma}
  The efficient influence function for
  \[\Psi_0(X, Y, Z) = \E\left[L(y, \hat{y}(X, Z))\right]\]
  is:
  \begin{align*}
&\phantom{AA}\psi_0(X, Y, Z)\\
&= (Y - \hat{y}(X,Z)) \int L'(y,\hat{y}(X,Z))  P(y|X,Z) dy  \\&
+ L(Y,\hat{y}(X,Z)) - \Psi_0(P).
  \end{align*}
\end{lemma}

\begin{lemma}
  The efficient influence function for  
  \[\Psi^{C }_0(X, Y, Z) = \E\left[L(y, \hat{y}(X^C, Z))\right]\]
  is:
  \begin{align*}
\psi^{C }_0(X, Y, Z) &= \int L'(y, \hat{y}(X, Z))(Y- \hat{y}(X, Z)) p(y| Z)dy \\ 
  & + \int L\left( y, \hat{y}(X, Z)\right)p(y|Z)dy \\
  & - \int L\left( y, \hat{y}(x, Z) \right)p(y|Z) p(x|Z)dxdy\\
  & + \int L\left(Y, \hat{y}(x, Z)\right)p(x|Z)dx - \Psi^{C }_0(P).
  \end{align*}
\end{lemma}

The algorithm for calculating conditional permutation importance  is shown in Algorithm~\ref{alg: conditional permutation importance calculation}, following the methodology outlined in Section \ref{sec: proposed methodology}.  

Implementing a TL update for CPI requires, in addition to and estimate of $\hat{y}(X,Z)$, which we obtain from $\hat{f}$, auxiliary estimates of $p(y|Z)$ and $p(x|Z)$.  In practice, we implement these via a weighted empirical distribution on $I_1$ and calculate the integrals above via Monte Carlo simulation. In particular, we fit a random forest (RF) to predict $X$ or $Y$ from $Z$ and  use OOB data to derive the conditional distributions for $p(y|Z)$ and $p(x|Z)$ from the tree kernel defined by in-leaf proximities, following \cite{lu2021unified}. We express this as $P(y = Y_i|Z) =  w_i(Z)$ where the $w_i(Z)$ are initially obtained the frequency with which $Z_i$ appears in the same leaf as $Z$ across trees in the RF for which $(Y_i,Z_i)$ is out of bag.  To construct an updated TL distribution we simply need to multiply the weights $w_i(Z)$ by $(1 + \hat{\epsilon}) \hat{\psi}$.

We can easily generalize this algorithm into $K-$folds rather than a single split, which would result in the same asymptotic result; choosing \(K = 10\) produces to a more numerically stable result. We refer the readers to \cite{smith2023application} for a more comprehensive review on the selection of folds for targeted learning.

\begin{algorithm}[!ht]
  \begin{algorithmic}[1]
    \Require $\{Y_{i}, X_i, Z_{i}\}$ for $i = 1, \dots, n$, \(I_1, I_2, I_3\) such that \(I_1 \cup I_2\cup I_3 = \left\{ 1, \dots, n \right\}\) and \(I_1 \cap I_2 \cap I_3 = \emptyset\). 
    \State Train an initial estimate $\hat{f}_{I_1}$.

    \State Estimate \(\hat{P}(x|z), \hat{P}(y| z)\).

    \For {each iteration $t$}
  \State Sample from \(\hat{P}(x|z), \hat{P}(y| z)\), denoted as \(\left\{ X^*_j \right\}_{j = 1, \dots, m}\) and \(\left\{ Y^* \right\}_{k = 1, \dots, m}\) respectively.
  \State Calculate 
  \[\hat{\Psi}_{I_2, 0}^{C} = \frac{1}{|I_2|} \sum_{i\in I_2} (Y_i-\hat{f}(X_i^C, Z_i))^2. \]
  and
  \begin{align*}
    &\quad\hat{\psi}_{I_2, 0}^{C}(X_i,Y_i,Z_i;\hat{P}) \\
    & \hspace{1cm} =\frac{1}{m}\sum_{j = 1}^n  L (Y_i,\hat{f}(X_j^*,Z_i))\\
    & \hspace{1.5cm}- \frac{1}{m^2}\sum_{j = 1}^n \sum_{k = 1}^n L(Y_k^*,\hat{f}(X^*,Z_i))\\
    & \hspace{1.5cm}+ \frac{1}{m}\sum_{k = 1}^n  L(Y^*_k, \hat{f}(X_i,Z_i)) 
  \end{align*}
  \State  Find $\hat\epsilon$ to maximize the likelihood of $\sum_{i\in I_2}c(\hat\epsilon)\hat{P}(X_i, Y_i, Z_i)(1+\hat\epsilon \hat{\psi}_{I_2, 0}^{C}(X_i, Y_i, Z_i;\hat{P}))$.
  \State Update $\hat{P} = c(\hat\epsilon)(1+\hat\epsilon \hat{\psi}_{I_2, 0}^{C})\hat{P}$ 
  \EndFor
  \State Repeat the above iteration until convergence.
  \State \textbf{Return:} \(\hat{\Psi} (\hat{f}_{I_1}, {P}_{\varepsilon^{k_n}})\) and variance \(\sqrt{\frac{1}{n}\sum_{i\in I_3}\hat{\psi}_{I_2, 0}^{C} }\) based  on \(I_3\).
    \end{algorithmic}
    \caption{Conditional permutation calculation on \(I_1\) with mean squared error loss}\label{alg: conditional permutation importance calculation}
\end{algorithm}

\subsection{Why conditional permutation importance?}\label{sec: why cpi?}
We  explore conditional permutation importance for two reasons: CPI provides an orthogonal factorization and CPI avoids density ratio estimation, as mentioned in \cite{verdinelli2024decorrelated}.

{\it Example: CPI factorization} 

Traditional targeted learning relies on factorizing the influence function into orthogonal components in order to conduct updates, as seen in average treatment effect estimation \citep{van2011cross}.

For conditional permutation importance we consider components corresponding to \(P_{X, Y|Z}\) and \(P_{Z}\). Then, for mean squared error, we have:
\begin{itemize}
  \item \(P_{X, Y|Z}\): 
  \begin{align*}
\psi^{C}_0(X,Y|Z) &= \int L'(y, \hat{y}(X, Z))(Y- \hat{y}(X, Z)) p(y| Z)dy\\
& + \int L\left( y, \hat{y}(X, Z)\right)p(y|Z)dy \\
& - 2\int L\left( y, \hat{y}(x, Z) \right)p(y|Z) p(x|Z)dxdy\\
& + \int L\left(Y, \hat{y}(x, Z)\right)p(x|Z)dx 
  \end{align*}
  \item \(P_{Z}\): 
\begin{align*}
  \psi^{C}_0(Z) 
  & =  \int L\left( y, \hat{y}(x, Z) \right)p(y|Z) p(x|Z)dxdy\\
  &- \Psi^{C}_0(P).
\end{align*}
  \end{itemize}
  Notice that for \(P_{Z}\) the empirical log-likelihood of the data points \(Z\) is  already maximized at the empirical distribution and thus no update is needed. 

  For \(P_{X, Y|Z}\), we employ the iterative update methodology, which follows a similar structure as algorithm \ref{alg: conditional permutation importance calculation}, but with a much simpler form. 
 
  As a contrast, we consider the efficient influence function of the LOCO importance. Following the similar construction as conditional permutation importance, for LOCO importance, we have:
  \[\Psi^{d}(X, Y, Z) = \Psi^{d}_0(Y, Z) - \Psi_0(X, Y, Z),\]
  where \(\Psi^{d}_0(Y, Z) = \E\left[L(y, \hat{y}(Z))\right]\). The corresponding influence function is:
  \begin{lemma}[\cite{williamson2020efficient}]
Let \(\Psi^{d}_0(X, Y, Z) = \E\left[L(y, \hat{y}(Z))\right]\), the efficient influence function is:
\begin{align*}
  \psi^{d}_0(X, Y, Z) &= (Y - \hat{y}(Z)) \int L'(y,\hat{y}(X,Z))  P(y|Z) dy\\&  + L(Y,\hat{y}(Z)) - \Psi^d_0(P).
\end{align*}
  \end{lemma}

  Compared to CPI, the estimation of the LOCO involves the estimation of two models \(\hat{y}(Z), \hat{y}(X,Z)\), where \(\hat{y}(Z)\) is based on the retraining of a new model based on perturbed data. This is described as having variational dependent models and thus need a more subtle treatment.

  Finally, we also examine the influence function of the traditional permutation importance metric.

\begin{lemma}
Let 
\[\Psi^{\pi L}_0(X, Y, Z) = \E\left[L(y, \hat{y}(X^\pi, Z))\right].\]
The efficient influence function is:
\begin{align*}
&\phantom{AA}\psi^{\pi L}_0(X, Y, Z) \\
& = (Y - \hat{y}(X, Z))\int L'(y,\hat{y}(X,Z)) \frac{P(X)P(y,Z)}{P(X, Z)} dy \\ 
&+  \int L(Y,\hat{y}(x',Z))P(x')dx'\\
&+ \int L(y,\hat{y}(X,z))P(y,z)dydz - 2 \Psi^{\pi L}_0(P),
\end{align*}
where \(X^\pi \sim X\), and \(X^\pi \perp X\).
\end{lemma}

We note that the performance of the density ratio estimation in the first term can be unstable, due to extrapolation and inherent low density at certain regions -- a problem that also applies to the decorrelated LOCO as mentioned in \cite{verdinelli2024decorrelated}.

\section{Theoretical results}\label{sec: theory}
In this section, we present the theoretical results of our methodology.
To formally introduce the theoretical results, we start with a brief introduction to a few concepts that would be helpful in developing our method. We start with the efficiency theory and methodology of targeted learning in section \ref{sec: raltl}, then we present the theoretical result for the estimator obtained in algorithm \ref{alg: conditional permutation importance calculation}.

\subsection{Efficiency theory and targeted learning}\label{sec: raltl}
By considering the variable importance as {\it general parameter} \(\Psi: P \to \R, P \in \mathcal{M}\), where \(\mathcal{M}\) is the class of nonparametric distributions, our aim is to find a ``good'' estimator of the true value \( \Psi(P^*)\), and then construct the corresponding confidence interval.  We define ``good'' using three criteria:
\begin{itemize}
\item Consistency: we would like to construct an estimator that is statistically consistent, which can be guaranteed by {\it asymptotically linearity}.
\item Robustness: we would like to construct an estimator that is robust to small perturbations of the data distribution, which can be guaranteed with {\it regularity}.
\item Efficiency: we hope to have an estimator that has minimum-possible variance given the available data, which will be ensured by the TL methodology, based on the two other requirements.
\end{itemize}
Building upon these three objectives, our goal is to construct an efficient, regular, and asymptotically linear estimator. In the following sections, we will rigorously define our concepts and then introduce the targeted learning methodology to construct such an estimator.

\subsubsection{Regular Asymptotically Linear (RAL) Estimators} \label{sec: raldef}
To begin with, we at least hope we can estimate with guaranteed consistency. One such class of estimators is the {\it asymptotically linear} estimators, where classical {\it asymptotically linear} estimators for parametric models include maximum likelihood estimation and generalized method of moments under mild conditions. In addition to the fact that the influence function characterizes the first-order term of a pathwise differentiable estimand, it also determines the asymptotic distribution of asymptotic linear estimator. Formally, asymptotically linear estimator is defined as:
\begin{definition}\label{def1: asymptotically linear}
An estimator sequence \(\{\hat{\Psi}_n(P^*)\}\) is said to be asymptotically linear with influence function \(\psi \in L_2^0(P^*)\) at distribution \(P^*\) if 
\begin{align*}
\sqrt{n}(\hat\Psi_n(P^*) - \Psi(P^*)) -  \frac{1}{\sqrt{n}}\sum_{i = 1}^n \psi(X_i) = o_P(1).
\end{align*}
\end{definition}

We note that when the influence function is the same as {\it efficient influence function}, the asymptotically linear estimator is efficient.

\subsubsection{Tangent Space}\label{sec: tangent space}
The {\it tangent space} characterizes the collection of possible functions to locally construct a path between distributions, defined by score function \(h = \frac{d}{d\varepsilon}\log dP^*\big |_{\varepsilon = 0}\) and their linear combinations at distribution \(P^* \in \mathcal{M}\) \citep{bickel1993efficient,van2000asymptotic}. 
Formally, {\it tangent space} is defined as:
\begin{definition}
  Let \(\{V_1, \dots, V_n\}\) denote the collection of score functions of \(P^* \in \mathcal{M}\), then the tangent space \(\dot{P^*}\) of \(P^*\) is defined as the linear span of \({V_1, \dots, V_n}\).
\end{definition}

For the class of nonparametric distributions \(\mathcal{M}\), the tangent space is \(\dot{\mathcal{P}^*}\coloneq L_2^0(P^*)\) \citep{bickel1993efficient}. 

With the tangent space defined, we can say that a sequence of estimators \(\hat\Psi_n\) at \(P^*\) is {\it regular} if there exists a probability measure \(L\) such that:
\begin{align*}
  \sqrt{n}\left(\hat\Psi_n - \Psi\left(P_{1/\sqrt{n},g}\right)\right) \overset{{P_{1/\sqrt{n},g}}}{\rightsquigarrow} L, \quad \mbox{for each } g \in \dot{\mathcal{P}^*},
\end{align*}
where \(P_{1/\sqrt{n},g} = (1 + \frac{1}{\sqrt{n}}g) P^*\).

\subsection{Asymptotic Results}
In this section, we outline the assumptions necessary to establish the efficiency of our final estimator and then present our main theorem. 
\begin{assumption}[Convergence]\label{assumption: convergence}
Let \(k_n\) denote the number of iterations until the algorithm converges. Assume that there exists \(k_n = k(\hat{P}) >0\) such that \(P(k(\hat{P}) < k_0) \to 1\) for some \(k_0 \equiv k(P^*)\) and 
\begin{align*}
   \frac{1}{|I_2|}\pr_{n, I_2}\psi_{P_{\varepsilon}^{k_n}} = o_P(1/\sqrt{n}),
\end{align*}

The same equation holds if we consider the empirical distribution of \(I_3\). In addition, we assume that the \(k_0-\)th step of estimate \( P_{\varepsilon^{k_0}}\) converges to \(P^*\) almost surely, where \(P^*\in \mathcal{M}\) is the least favorable model.
\end{assumption}
This assumption is standard for cross-validation TL \citep{van2011cross}. Here we assume that the algorithm will converge in at most the \(k_0\) steps and that the efficient influence function will be small. In addition, the assumption ensures the limiting distribution is within the nonparametric model class and the first-order optimality. Lastly, Assumption~\ref{assumption: convergence} implicitly places an assumption on the initial estimator, as a poorly chosen initial estimator could result in divergence. In practice, initial estimators based on either the plug-in or Z-estimation approach have been shown to perform well.

\begin{assumption}[Differentiability and Optimality]
\label{assumption: derivative and second order term}
Given a variable importance metric \(\Psi\), we assume that it is pathwise differentiable for the class of nonparametric distributions \(\mathcal{M}\). In addition, the von Mises expansion satisfies:
\begin{align*}
\Psi(\hat{f}_{I_1}, \hat{P}) - \Psi(f^*, P^*) 
&= \int \psi(\hat{f}_{I_1}, \hat{P}) d(\hat{P} - P^*) \\&+ O_P(\|\Psi(\hat{f}_{I_1},\hat{P}) - \Psi(f^*, P^*)\|^2),
\end{align*}
where \(\hat{P} \in \mathcal{M}\), \(f^* \equiv \hat{y}\), and \(\hat{y}\) is defined as in equation \ref{equation: definition vi}.
\end{assumption}

This assumption restricts the differentiability of the variable importance measure and imposes an assumption on the asymptotic performance of the second-order remainder term, originating from \cite{van2011cross}. Together with Assumption \ref{assumption: convergence}, the above two results guarantee the asymptotic efficiency of the TL estimator. 

\begin{remark}
We note that assumption \ref{assumption: derivative and second order term} functions in a similar manner as the (A1) and (B1) given in \cite{williamson2023general} or Assumption 5 of \cite{wei2023efficient}. In both cases, the aim is to control the second-order term. By considering the second-order term as the order of the bias directly, the proof is greatly simplified.
For conditional permutation variable importance, we can alternatively have the order of \(\E_{I_2}\left[\| \hat{p}(y|z) - p(y|z)\right] \|, \) \(\E_{I_2}\left[\|\hat{p}(x|z) - p(x|z) \|\right],   \E_{I_2}\left[\| \hat{f}(x, z) - f^*(x, z) \|\right] \)  be \(o_{P}(n^{-1/4})\), where \(f\) is the estimator and \(\hat{p}(x|z)\) is the estimator of density \(p(x|z)\). A similar assumption can be defined for \(I_3\) as well.
\end{remark}

\begin{assumption}[Consistency]
\begin{align*}
\int \left( \Psi(\hat{f}_{I_1}, {P}^*) - \Psi({f}^*, P^*)\right)^2dP^* = o_P(1)
\end{align*}
\end{assumption}
This assumption ensures the consistency of the plug-in estimator, which is also given in \cite{van2011cross}. Without such, the result wouldn't be efficient. 

\begin{assumption}[Sample-Splitting]\label{assumption: samplesplitting}
Let \(\varepsilon_{*}^{j}\) be the limit of \(\varepsilon_n^{j}\), that is \(\varepsilon_n^{j} \converge[P] {\varepsilon}_{*}^j\) for $j \in 1,\ldots,k_0$. 
We assume that the final efficient influence function \(\psi_{P_{\varepsilon^{k_n}_n}}\) is estimated from \(I_1, I_2\), independent from empirical measure of \(I_3\), denoted as \(\pr_{n, I_3}\). And our final estimator is obtained through \(I_3\). To ensure the consistency, we assume that \(\sup_{j \leq k_0} \|\psi_{P_{\varepsilon^{j}_n}} - \psi_{P^*} \| = o_P(1)\).
\end{assumption}
\begin{remark}
  Assumption \ref{assumption: samplesplitting}  is specifically designed to address the empirical process term. In particular, we introduce an additional subset of the data, \(I_3\), to ensure independence between the efficient influence function and the final estimator. Although this approach differs from the classical method described by \cite{10.1111/ectj.12097}, it is necessitated by the iterative nature of our procedure, in contrast to their one-step framework.
\end{remark}
As an alternative to Assumption 4, Donsker assumptions similar to A2 of \cite{van2011cross} can also be made to establish the asymptotic results, which we refer to as Assumption 5; details can be found in the supplementary materials.
\begin{theorem}\label{thm: mainthm}
  Assume that Assumptions 1-3 hold, and Assumption 4 or 5 hold. Our final estimator \(\hat{\Psi} (\hat{f}_{I_1}, {P}_{\varepsilon^{k_n}_n})\) is asymptotically linear and satisfies:
  \begin{align*}
\hat{\Psi} (\hat{f}_{I_1}, {P}_{\varepsilon^{k_n}_n}) - \Psi(P^*) = \pr_n \psi_{P^*} + o_P(1/\sqrt{n}),
  \end{align*}
  where \(\psi(P^*)\) is the efficient influence function.
\end{theorem}

Note that the above three assumptions are defined for \(\hat{f}_{I_1}\), where similar assumptions can easily be defined for \(\hat{f}_{I_2}, \hat{f}_{I_3}\). As a result, we can also average over swapping the rolls of subsets and \( \bar{\Psi} \equiv 1/3(\hat{\Psi} (\hat{f}_{I_1}, {P}_{\varepsilon^{k_n}_n}) + \hat{\Psi} (\hat{f}_{I_2}, {P}_{\varepsilon^{k_n}_n}) + \hat{\Psi} (\hat{f}_{I_3}, {P}_{\varepsilon^{k_n}_n}) ) \) obtains the same asymptotic results while achieving better sample efficiency.

Our result implies that our estimator is asymptotically normal and efficient. We note that under similar assumptions, the plug-in estimator with bias correction may obtain the same asymptotic performance, yet preserve worse finite data performance. These asymptotic results can also be easily extended to \(K-\)fold case and the asymptotic performance would remain the same.

\begin{remark}
Compared to \cite{van2011cross}, the only difference in assumptions is that we do not impose Donsker conditions on $\hat{f}$ by considering the sample splitting methodology if we adopt Assumptions 1-3 and 4. Relative to \cite{williamson2023general}, we additionally imposed Assumption, Assumption \ref{assumption: convergence} to ensure the stochastic convergence of the algorithm.
\end{remark}

\begin{figure}[!ht]
\centering
Bias \\
\includegraphics[width=0.46\textwidth]{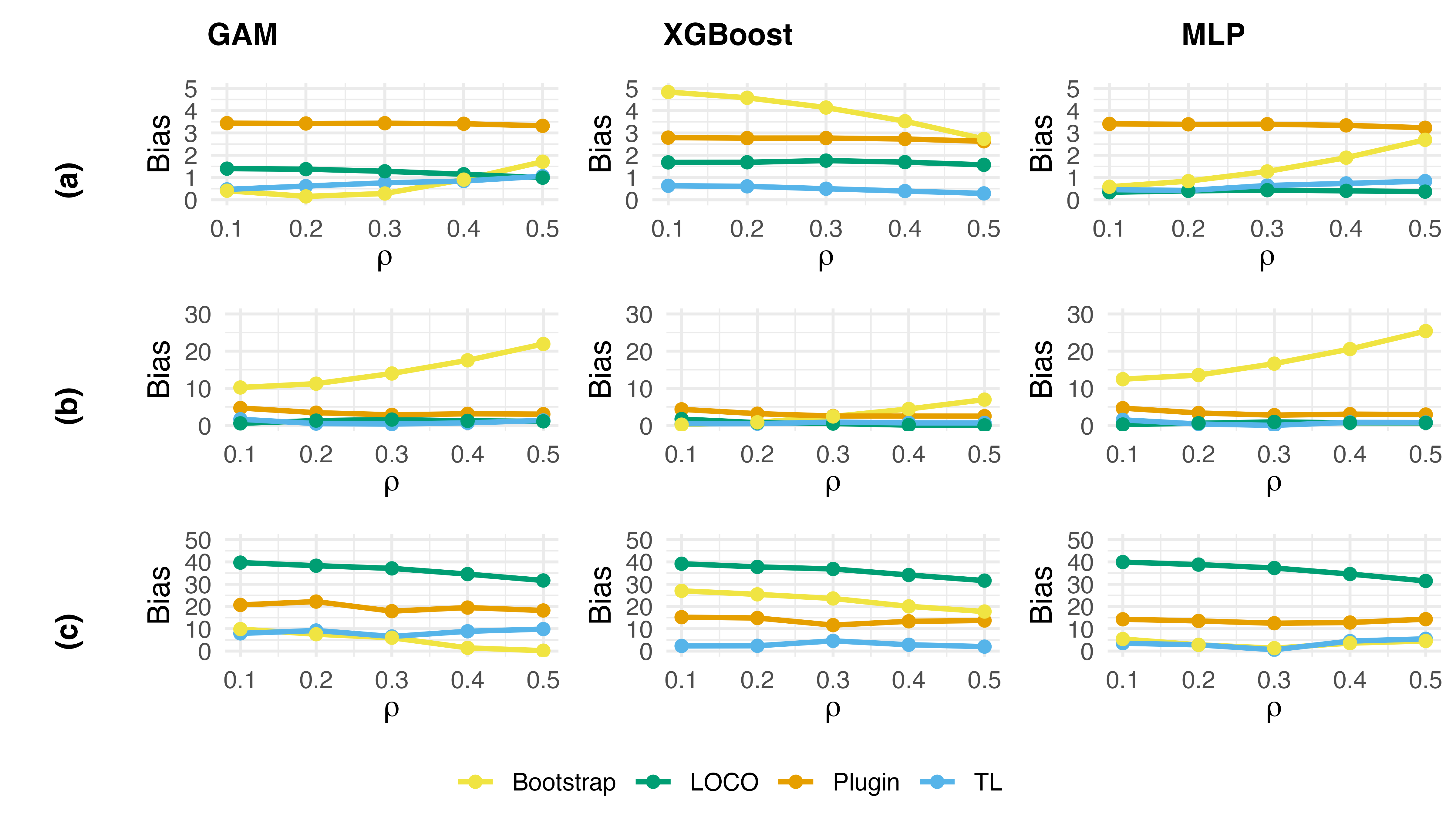} \\
Coverage \\
\includegraphics[width=0.46\textwidth]{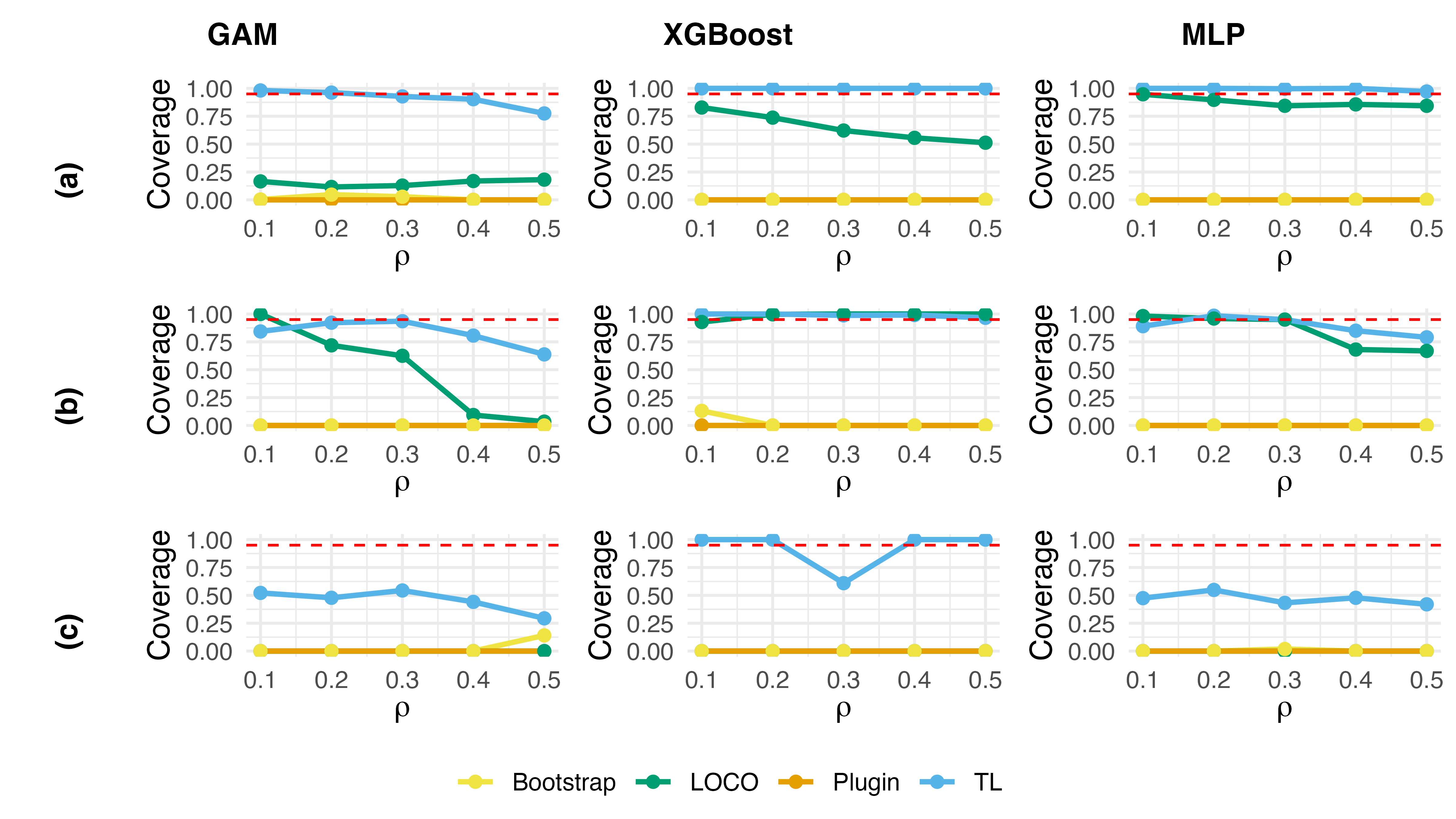}\\
CI Length \\
\includegraphics[width=0.46\textwidth]{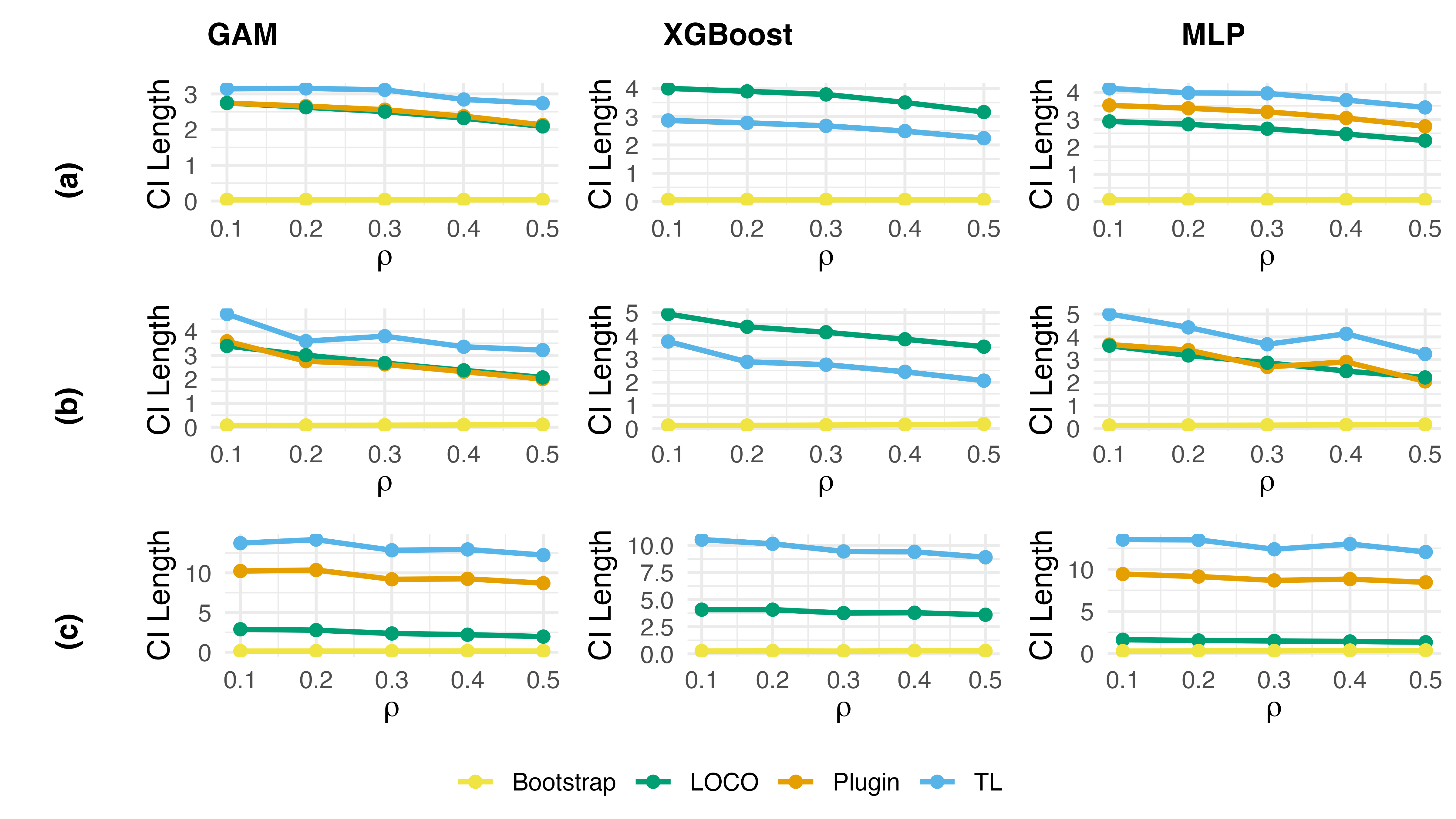}
\caption{Bias, Coverage and Length of Confidence Intervals of targeted learning and plug-in estimators using three different initial estimators: General Additive Model (left), XGBoost (middle), and Multi-Layer Perceptron (right) based on 240 simulated data, each with 1000 observations}
\label{figure: main simulation}
\end{figure}

\section{Simulation Study}
In this section, we present the simulation results, starting by comparing the bias of our estimator with that of the plug-in estimator \cite{strobl2008conditional}, as well as examining a bootstrap correction to the bias. 

To construct the initial estimator \(\hat{f}\), we consider three models: Generalized Additive Model (GAM) via \texttt{pyGAM} package, Multi-Layer Perceptron (MLP) implemented in \texttt{scikit-learn}, and eXtreme Gradient Boosting (XGBoost) from \texttt{xgboost} package. To mitigate the impact of hyperparameter tuning, we employ the default parameters for both GAM and XGBoost; additional technical details are provided in the supplementary materials. 

For each simulation setting, we generate \(1{,}000\) observations and repeat the entire procedure \(240\) times. The outcome \(y_i\) is generated following one of the following three designs:
\begin{enumerate}[label=(\alph*)]
\item \(\displaystyle y_i = 3 x_{i1} + \epsilon_i\), following \cite{verdinelli2024feature}.
\item \(\displaystyle y_i = 3x_{i1} + x_{i2} + x_{i3} + x_{i4} + x_{i5} + 0x_{i6} + 0.5x_{i7} + 0.8x_{i8} + 1.2x_{i9} + 1.5x_{i10} + \epsilon_i,\) following \cite{hooker2021unrestricted}.
\item \(\displaystyle y_i = 10\sin(x_{i1}) + 10\cos(x_{i2}) + 3\,x_{i3}x_{i6} + 3x_{i10} + \epsilon_i.\)
\end{enumerate}

In all cases\(\epsilon_i \sim \mathcal{N}(0,1)\). Following \cite{hooker2021unrestricted}, the 10-dimensional covariate vector \(X\) is sampled from a multivariate normal distribution with mean vector \(0\) and covariance matrix \(\Sigma\). For setting (a), the covariance matrix is defined as \(\Sigma_{ii} = 1, \Sigma_{12} = \Sigma_{21} = \rho\) and 0 otherwise. In settings (b) and (c), the covariance matrix is \(\Sigma_{ij} = \rho^{|i-j|}\).

Consistent with \cite{verdinelli2024feature} and \cite{hooker2021unrestricted}, we vary \(\rho \in \{0.1, 0.2, 0.3, 0.4, 0.5\}\) to examine the performance of the proposed algorithm under different correlation structures. For settings (a) and (b), the true value is derived theoretically, while for setting (c) it is approximated via Monte Carlo integration.

As illustrated in Figure~\ref{figure: main simulation}, our estimator consistently exhibits lower bias compared to the plug-in estimator and offers a consistently good estimator compared to LOCO and bootstrap estimators. The proposed estimator generally achieves superior coverage compared to the plug-in estimator. This improvement is primarily attributable to a reduction in bias, as evidenced in Figure~\ref{figure: main simulation}. In particular, the plug-in estimator exhibits considerably higher bias than the TL estimator, which in turn leads to substantially compromised coverage. Moreover, when the correlation is strong, the performance of both the initial estimator \(\hat{f}\) and the conditional density estimator deteriorates, potentially violating Assumption~\ref{assumption: convergence}. Although the confidence interval (CI) length for our proposed estimator is slightly longer, it achieves significantly better coverage than the plug-in estimator, indicating superior overall performance. For the XGboost model, the CI lengths for both estimators are nearly identical, making the corresponding lines in the plots indistinguishable. A table of computational costs is provided in the supplementary material.

\subsection{Real World Data Application}
\subsubsection{Bike sharing}

In our real-data application, we examine the variable importance scores for the hourly bike share dataset obtained from the UCI repository \citep{bike_sharing_275}. We employ XGBoost to generate the initial estimates, and the results are presented in Figure~\ref{figure: bikeshare}. 

From the plot, we see that \texttt{workingday} and \texttt{yr} sit well above the rest in terms of importance, suggesting they explain a larger share of the variability in the response than other predictors. Meanwhile, features like holiday and \texttt{weathersit} occupy the next tier of influence, although their bars are noticeably shorter. At the lower end, variables such as month and \texttt{wdspd} barely rise above zero, implying they may add little explanatory power. A noteworthy takeaway is how \texttt{Temp} and \texttt{atemp} rank surprisingly low, despite one might expect temperature-related variables to matter more. Hence, even though many features cluster in a middle range of importance, the disagreements at the extremes illustrate why a nuanced approach to screening (beyond raw importance scores alone) is often necessary for sound statistical analysis.

\begin{figure}[!ht]
  \centering
  \includegraphics[width=.5\textwidth]{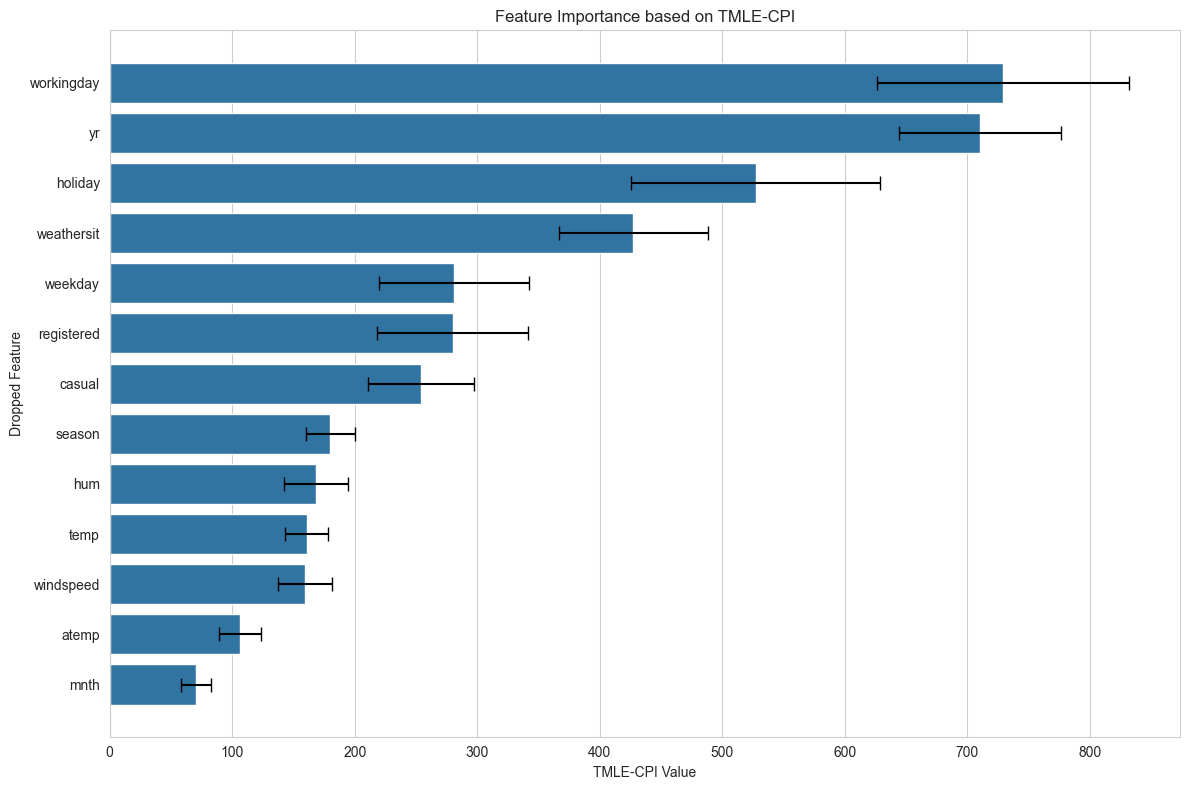}
  \caption{Conditional variable importance scores for the hourly bike share dataset, obtained using TL with an XGBoost-based initial estimate.}
  \label{figure: bikeshare}
\end{figure}

\subsubsection{Wine quality}
In addition, we included the wine quality dataset to illustrate the application of our method in classification settings through the wine quality dataset. We employ random forest to generate the initial estimates.

\begin{figure}[!ht]
  \centering
  \includegraphics[width=.4\textwidth]{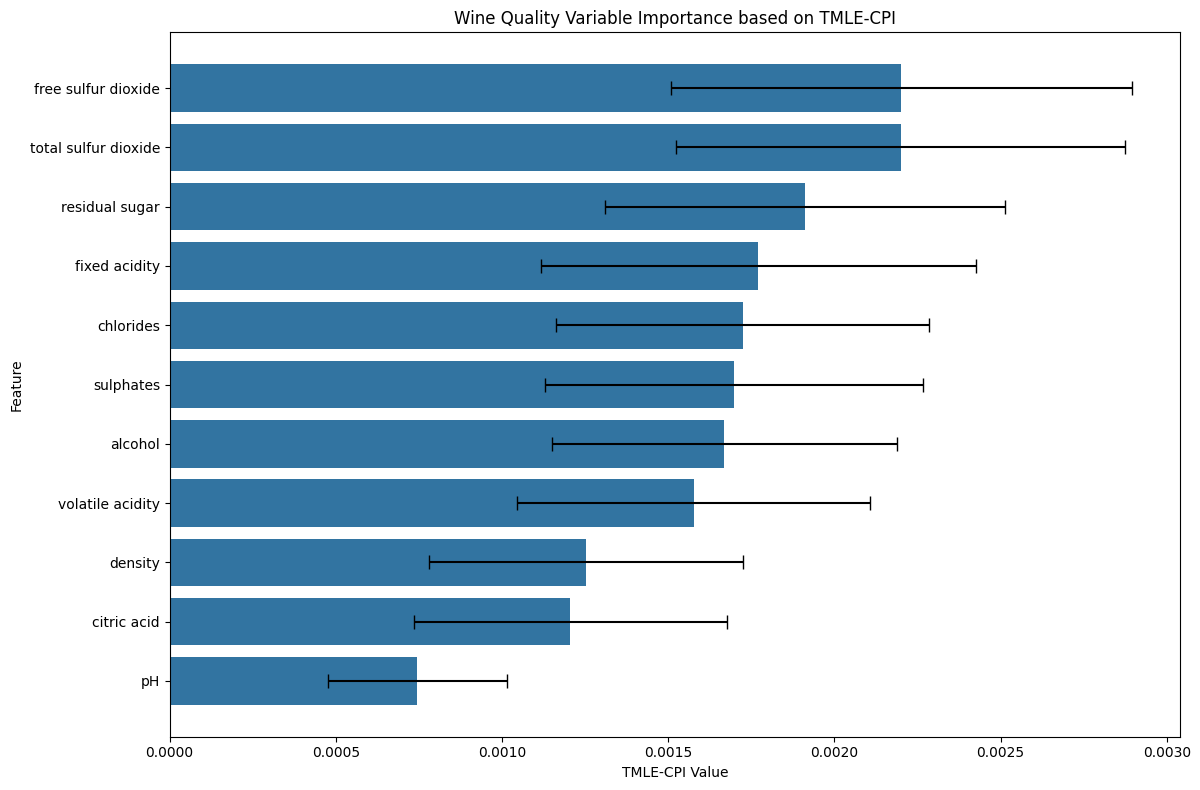}
  \caption{Conditional variable importance scores for the wine quality, obtained using TL with an Random Forest-based initial estimate.}
  \label{figure: winequality}
\end{figure}
From the conditional permutation importance plot, we see that \texttt{free sulfur dioxide} and \texttt{total sulfur dioxide} sit well above the rest in terms of importance, suggesting they explain a larger share of the variability in wine quality than other chemical measures. Meanwhile, \texttt{residual sugar} and \texttt{fixed acidity} occupy the next tier of influence, although their bars are noticeably shorter. In the middle range, variables such as \texttt{chlorides}, \texttt{sulphates}, \texttt{alcohol}, and \texttt{volatile acidity} cluster with moderately high importance, indicating their meaningful but not dominant contribution. Toward the lower end, features like \texttt{density} and \texttt{citric acid} display only modest importance, while \texttt{pH} barely rises above zero, implying it adds little explanatory power in this context. A noteworthy takeaway is how preservative-related variables dominate the ranking even though one might expect acidity or alcohol content to matter more strongly. Here we note that width of our confidence intervals suggest that the data only provide a highly uncertain ranking of variable importance.

\section{Conclusion}

In this paper, we study uncertainty quantification in IML using the targeted learning framework, illustrated through conditional permutation importance. Under mild assumptions, our methodology achieves asymptotic efficiency, maintains comparable computational complexity, and delivers improved finite-sample accuracy. 

Future work includes developing methodology for estimating the overlap model as mentioned in Section~\ref{sec: why cpi?}, since we cannot factorize the subspace into orthogonal ones. It is also interesting to consider problems involving density ratios, which might be more approachable using methods that bypass the calculation of influence function, such as \cite{cho2023kernel,van2024automatic}.

\clearpage
\bibliography{vitmle}

\newpage

\onecolumn

\title{Targeted Learning for Variable Importance\\Supplementary Material}
\maketitle

\appendix
Here, we present the proof of our main theorem, along with additional simulation results that could not be included in the main text due to space constraints. We begin by providing more details of simulation results.

\section{Simulations And Real World Data}
\subsection{Technical Details: Parameter Selection implementation details}
In this section, we mostly adopt default parameter settings to maintain consistency with \texttt{R} and to avoid any performance improvements arising from parameter tuning. Specifically, for the Generalized Additive Model (GAM), we use the default parameters. For the Random Forest model employed in conditional density estimation, we utilize the default settings provided by \texttt{randomForest} package in \texttt{R}. In the case of XGBoost, we adjust the number of estimators to match that of the Random Forest. Lastly, for the MLP regressor, we design a two-layer network with 64 neurons in the first layer and 32 neurons in the second layer, applying the ReLU activation function for non-linearity and the Adam optimizer for weight updates. Training is configured to run for a maximum of 3000 iterations to ensure convergence.

To estimate densities $p(x|Z)$ and $p(y|Z)$ we started by fitting a random forest of $B$ trees to predict each from $Z$ using $I_1$.  We then use this to provide an initial estimate based on a weighted empirical distribution from $I_1$.  $P(y=Y_i|Z) = w_i(Z)$ calculated as the fraction of trees for which $(Z_i,Y_i)$ was out-of-bag, in which $Z$ and $Z_i$ fall into the same leaf.  In the targeted learning update, the $w_i$ are multiplied by $(1+ \hat{\epsilon}) \psi_{\hat{P}}(X_i,Y_i,Z_i)$ allow us to keep track of the updated distribution, and later apply it to $I_3$. The same procedure was employed to update $p(x|Z)$.

For settings (a) and (b), theoretical calculations based on Theorem 2 in \cite{hooker2021unrestricted} show that the true value is \(9(1-\rho^2)\). In contrast, due to the complexity of the nonlinear model in setting (c), the true value is estimated via Monte Carlo integration.

\subsection{Computational cost comparison}
\begin{table}[h!]
  \centering
  \caption{A Comparison of Computational Runtimes}
  \label{tab:comp_costs_simple}
    \begin{tabular}{l S[table-format=2.2]}
      \toprule
      \textbf{Method} & {\textbf{Runtime (\si{\second})}} \\
      \midrule
      Plug-in   & 1.52 \\
      TL        & 3.47 \\
      Bootstrap & 57.20 \\
      \bottomrule
    \end{tabular}
    \par
    \small{\textit{Note:} Runtimes were measured on a machine with eighty Intel® Xeon® Gold 6230 CPUs @ 2.10GHz.}
\end{table}

\section{Proof}
To facilitate these proofs, we first introduce the necessary notation.
\subsection{Notations}
We use \(\pr_n\) to denote the empirical measure, that is, suppose \(f: \mathcal{X} \to \R \), \(\pr_n(f) = \frac{1}{n}\sum_{i=1}^n f(X_i)\). In contrast, we use \(\pr\) to denote the probability measure, that is, \(\pr(f) = \int f(X) d\pr\). \(L_2^0(P)\) denotes the collection of functions such that \(P f =0 \) and \(Pf^2 < \infty\). \(O_P\) and \(o_P\) are used as follows: \(X_n = O_P(r_n)\) denotes \(X_n/r_n\) is bounded in probability and \(X_n = o_P(r_n)\) indicates \(X_n/r_n \converge[P] 0\), respectively. Additionally, we denote the \(L_2(P)\) norm as \(\|\cdot\|\). In addition, we denote the conditional mean of \(y\) given \(x\) as \(\hat{y}(x) \equiv \E[y|X = x]\). We assume that the dataset \(I\) is divided into three mutually exclusive sets \(I_1, I_2, I_3\). To facilitate the calculation of the efficient influence function, we denote \(\delta_{O}(o)\) as the Dirac delta function with respect to \(O\), i.e., the density of an idealized point mass at \(O\), which equals zero everywhere except at \(O\) and integrates to 1.

\subsection{Efficient Influence Function}
We note that the derivation of the efficient influence function largely follows the work of \cite{hines2022demystifying}, adopting the “point mass contamination” methodology.

\begin{lemma}\label{lemma: reflosseif}
  Let 
  \[\Psi_0(X, Y, Z) = \E\left[L(Y, \hat{y}(X, Z))\right].\]
  The efficient influence function is:
  \begin{align*}
\psi_0(X, Y, Z) &= (Y - \hat{y}(X, Z)) \int L'(y, \hat{y}(X, Z)) P(y|X, Z) dy \\
&\quad + L(Y, \hat{y}(X, Z)) - \Psi_0(P).
  \end{align*}
\end{lemma}
\begin{proof}
We start by considering the integration form of the estimand, which can be expressed as:
\begin{align*}
\Psi_0(P)& = \E\left[ L(y, \hat{y}(x, z)) \right]\\
&= \int L(y, \hat{y}(x, z)) P(x, y, z) dxdydz.
\end{align*}

Then, by considering the product rule, we may have:
\begin{align*}
\psi_0(X,Y,Z) & = \int L' \left(y,\hat{y}(x,z)\right)\psi^{\hat{y}(x,z)}(X,Y,Z)P(x,z,y)dxdydz \\ & \hspace{2cm} + \int \int L \left(y,\hat{y}(x,z)\right)\left[ \delta_{XYZ}(x,y,z) - P(x,y,z) \right]dxdyz,
\end{align*}
where \(L'\left(y, \hat{y}(x,z)\right)\) is the derivative of \(L\left(y, \hat{y}(x,z)\right)\), and \(\psi^{\hat{y}(x,z)}(X,Y,Z)\) is the efficient influence function of \(\hat{y}(x,z)\).

From example 6 of \cite{hines2022demystifying}, we have \(\psi^{\hat{y}(x,z)}(X,Y,Z) = (Y - \hat{y}(x,z))\frac{\delta_{{X}, {Z}}(x, z)}{P(x, z)}\).

Then, we have:
\begin{align*}
 \psi_0(X,Y,Z) =   (Y - \hat{y}(X,Z)) \int L'(y,\hat{y}(X,Z))  P(y|X,Z) dy  + L(Y,\hat{y}(X,Z)) - \Psi_0(P)
\end{align*}
\end{proof}

\begin{lemma}
  Let 
  \[\Psi^{C}_0(X, Y, Z) = \E\left[L(Y, \hat{y}(X^C, Z))\right].\]
  The efficient influence function is:
  \begin{align*}
\psi^{C}_0(X, Y, Z) &= \int L'(y, \hat{y}(X, Z))(Y - \hat{y}(X, Z)) p(y|Z) dy \\
&\quad + \int L\left(y, \hat{y}(X, Z)\right) p(y|Z) dy \\
&\quad - \int L\left(y, \hat{y}(x, Z)\right) p(y|Z) p(x|Z) dx dy \\
&\quad + \int L\left(Y, \hat{y}(x, Z)\right) p(x|Z) dx - \Psi^{C}_0(P).
  \end{align*}
\end{lemma}
\begin{proof}
We can start off with the integral form as well, where we shall then get:
\begin{align*}
\Psi^C_0 = \int L(y, \hat{y}(x',z)) P(x'|z)P(x, y, z)dx'dxdydz
\end{align*}

Then, to consider the derivative, we may have:

\begin{align*}
\phi_0^C&= \int \left[ L(y , \hat{y}(x',z) \right]' 
\frac{p(x',z) p(z,y)}{p(z)}  \, dx' \, dz \, dy \, + \int L(y , \hat{y}(x',z') 
\left[ \frac{p(x',z) p(z,y)}{p(z)} \right]' \, dx' \, dy \, dz \, \\
& = R_1  + R_2
\end{align*}

From here, we start with the first term and adopt a similar treatment as Lemma~\ref{lemma: reflosseif}, which yields
\begin{align*}
R_1 &= \int L'(y , \hat{y}(x',z)) (Y - \hat{y}(x', z)) \frac{\delta_{X, Z}(x', z)}{p(x',z)}\frac{p(x',z) p(z,y)}{p(z)} \, dx\, dx' \, dz \, dy \,\\
&= (Y- \hat{y}(X, Z))\int L'(y , \hat{y}(X,Z)\frac{p(y,Z)}{p(Z)} dy
\end{align*}

Then, for the second term, we may have to consider decomposing it into three terms.

\begin{align*}
\int L\left( y , \hat{y}(x', z)\right) \frac{\delta_{X Z}(x', z) - p(x', z)}{p(z)}  p(y, z) \, dy \, dz 
= \int L\left( y , \hat{y}(x', Z) p(y | Z) \right) dy - \Psi^C_0 (P)
\end{align*}

{Also, we may have}:

\begin{align*}
&\int L\left( y , \hat{y}(x', z')\right)  \frac{\delta_{Z}(z) - p(z))}{p(z)^2} p(x',z) p(y, z) \, dx' \, dy \, dz \, dt \\
&= \int L\left( y , \hat{y}(x', Z)\right) p(x'|Z) p(y|Z)  dx \, dy - \Psi(P)\\
& = \int L\left( y , \hat{y}(x, Z)\right) p(x|Z) p(y|Z)  dx \, dy - \Psi^C_0 (P)
\end{align*}

{Lastly, we then have}:

\begin{align*}
&\int L\left( y , \hat{y}(x, z)\right) \frac{p(x',{z})}{p({z})} \left( \delta_{Y,Z}(y, z) - p(y, z) \right)  dx \, dy \, dz \\
&= \int L\left( Y , \hat{y}(x', Z)\right) p(x'|Z)  dx'  - \Psi(P)\\
& = \int L\left( Y , \hat{y}(x, Z)\right) p(x|Z)  dx  - \Psi^C_0 (P)
\end{align*}

{Putting the three terms together, we shall have:}

\begin{align*}
R_2 &= \int L\left( y , \hat{y}(x', Z) p(y | Z) \right) dy  \\
&- \int L\left( y , \hat{y}(x, Z)\right) p(x|Z) p(y|Z)  dx \, dy  \\
&+ \int L\left( Y , \hat{y}(x, Z)\right) p(x|Z)  dx  - 2\Psi^C_0 (P)
\end{align*}

Putting everything together, we may then obtain the desired result.
\end{proof}

\begin{lemma}
  Let 
  \[\Psi^{\pi L}_0(X, Y, Z) = \E\left[L(Y, \hat{y}(X^\pi, Z))\right].\]
  The efficient influence function is:
  \begin{align*}
\psi^{\pi L}_0(X, Y, Z) &= (Y - \hat{y}(X, Z)) \int L'(y, \hat{y}(X, Z)) \frac{P(X) P(y, Z)}{P(X, Z)} dy \\
&\quad + \int L(Y, \hat{y}(x', Z)) P(x') dx' \\
&\quad + \int L(y, \hat{y}(X, Z)) P(y, Z) dy dz - 2 \Psi^{\pi L}_0(P),
  \end{align*}
  where \(X^\pi \sim X\) and \(X^\pi \perp X\).
\end{lemma}
\begin{proof}
Using a similar approach as Lemma~\ref{lemma: reflosseif}, we have
\begin{align*}
\psi^{\pi L}(X,Y,Z) &  = \int \psi^{\hat{y}(x',z)(X,Z)}L'(y,\hat{y}(x',z))P(x')P(y,z) dx'dydz \\ & \hspace{2cm} + \int L(Y,\hat{y}(x',Z))P(x')dx' + \int L(y,\hat{y}(X,z))P(y,z)dydz - 2 \Psi^{\pi L}(P) \\
& = (Y - \hat{y}(X,Z))\int L'(y,\hat{y}(X,Z)) \frac{P(X)P(y,Z)}{P(X,Z)} dy \\ & \hspace{2cm} + \int L(Y,\hat{y}(x',Z))P(x')dx' + \int L(y,\hat{y}(X,z))P(y,z)dydz - 2 \Psi^{\pi L}(P)
\end{align*}
\end{proof}
\subsection{Additional Assumptions}
\setcounter{assumption}{4}

We note that Assumptions \ref{assumption: donsker} below and Assumption \ref{assumption: samplesplitting} essentially play the same role in eliminating the empirical process term. In Assumption \ref{assumption: samplesplitting}, we used an additional share of data \(I_3\) to ensure the independence of the efficient influence function and the final estimator. Though this is different from the classical approach described by \cite{10.1111/ectj.12097}, we note that our method is iterative, whereas theirs is a one-step method. And Assumption~\ref{assumption: donsker} is a replicate Assumption A4 of \cite{van2011cross}.   Here we define $\overrightarrow{\varepsilon}_n^{k_0}$ to be the sequence of $\varepsilon^j_n$, padded with zeros if needed to create a $k_0$ vector. 

\begin{assumption}[Donsker Condition; A2 of Theorem 5 in \cite{van2011cross}]\label{assumption: donsker}
Let \(\varepsilon_{k_0}^*\) be the limit of \(\overrightarrow{\epsilon}_n^{k_0}\), that is, \(\overrightarrow{\epsilon}_n^{k_0} \converge[P] \varepsilon_{k_0}^*\). Condition on \(P_{n,I_2} \) and consider a class of measurable functions \(f\) estimated on \(I_1\):

\[
\mathcal{F}(P_{n,I_2}) \equiv \left\{ \psi(\hat{f}_{I_1}, P_{\varepsilon}) 
- \psi(f^*, P_{\varepsilon_{k_0^*}}): \varepsilon\right\},
 \]
where the set over which ${\epsilon}$ varies is chosen so that it is a subset of $\mathbb{R}^{k_0}$ and contains $\overrightarrow{\epsilon}_n^{k_0}$ with probability tending to 1. 
Define the subclasses
\[
\mathcal{F}_{\delta_n}(P_{n,I_2}) \equiv \left\{ f_\epsilon \in \mathcal{F}(P_{n,I_2}) : \|\epsilon - \varepsilon_{k_0}^*\| < \delta_n \right\}.
\]

If for deterministic sequence \( \delta_n \to 0 \), we have
\[
E\left\{ \text{Entro}(\mathcal{F}_{\delta_n}(P_{n,I_2})) \sqrt{P^* {F}(\delta_n, P_{n,I_2})^2} \right\} \to 0 \quad \text{as} \quad n \to \infty,
\]
where \( {F}(\delta_n, P_{n,I_2}) \) is the envelope of \( \mathcal{F}_{\delta_n}(P_{n,I_2}) \) and \(\text{Entro}(\mathcal{F}_{\delta_n}(P_{n,I_2}))\) is the entropy of \(\mathcal{F}_{\delta_n}(P_{n,I_2})\).

\end{assumption}
This condition is the same as A2 given in \cite{van2011cross}, to which we refer the reader for further details.

\subsection{Proof of Theorem~\ref{thm: mainthm}}
\begin{proof}
If Assumptions 1,2,3 and 4 are satisfied, this is exactly the same result as Theorem 5 of \cite{van2011cross}, and so will be the proof.

If Assumptions 1,2,3 and 5 are satisfied, the only thing we need to do is create a similar lemma as Lemma 2 of \cite{van2011cross}. We can start by considering the empirical process term, that is 
\begin{align*}
(\pr_{n, I_3} - \pr)\left(\psi(\hat{f}_{I_1}, P_{\overrightarrow{\varepsilon}^{k_n}_n} ) - \psi({f}^*, P^*) \right) 
\end{align*}
For the conditional variance of the term on \(I_3\), we have:
\begin{align*}
var\left( (\pr_{n, I_3} - \pr)\left(\psi(\hat{f}_{I_1}, P_{\overrightarrow{\varepsilon}^{k_n}_n} ) - \psi({f}^*, P^*) \right) \right) 
&= var\left( \pr_{n, I_3} \left(\psi(\hat{f}_{I_1}, P_{\overrightarrow{\varepsilon}^{k_n}_n} ) - \psi({f}^*, P^*) \right) \right) \\
&= \frac{1}{n} var(\psi(\hat{f}_{I_1}, P_{\overrightarrow{\varepsilon}^{k_n}_n} ) - \psi({f}^*, P^*) )\\
& \leq \frac{1}{n} \|\psi(\hat{f}_{I_1}, P_{\overrightarrow{\varepsilon}^{k_n}_n} ) - \psi({f}^*, P^*) \|\\
& = o_P(1/n)
\end{align*}
Then, by Chebyshev's inequality, we have:
\begin{align*}
 (\pr_{n, I_3} - \pr)\left(\psi(\hat{f}_{I_1}, P_{\overrightarrow{\varepsilon}^{k_n}_n} ) - \psi({f}^*, P^*) \right) = O_p \left(\sqrt{\frac{1}{n} \|\psi(\hat{f}_{I_1}, P_{\overrightarrow{\varepsilon}^{k_n}_n} ) - \psi({f}^*, P^*) \|}\right)
\end{align*}

We can then obtain the desired result by Assumption \ref{assumption: samplesplitting}, that is:
\begin{align*}
(\pr_{n, I_3} - \pr)\left(\psi(\hat{f}_{I_1}, P_{\overrightarrow{\varepsilon}^{k_n}_n} ) - \psi({f}^*, P^*) \right)  = o_P(1/\sqrt{n})
\end{align*}
The rest of the proof follows in the same manner as Theorem 5 of \cite{van2011cross}.
\end{proof}

\end{document}